\documentclass{article}

\usepackage{microtype}
\usepackage{graphicx}
\usepackage{booktabs}
\usepackage{multirow}
\usepackage{hyperref}

\usepackage[preprint]{icml2026}

\usepackage{amsmath,amssymb,amsthm,mathtools}
\usepackage{bbm}
\usepackage{enumitem}
\usepackage{algorithm}
\usepackage{algorithmic}

\theoremstyle{plain}
\newtheorem{theorem}{Theorem}[section]
\newtheorem{lemma}[theorem]{Lemma}
\newtheorem{corollary}[theorem]{Corollary}
\newtheorem{proposition}[theorem]{Proposition}
\theoremstyle{definition}
\newtheorem{definition}[theorem]{Definition}
\newtheorem{assumption}[theorem]{Assumption}
\theoremstyle{remark}
\newtheorem{remark}[theorem]{Remark}
\newtheorem{observation}[theorem]{Observation}

\newcommand{\E}{\mathbb{E}}

\newcommand{\Prob}{\mathbb{P}}

\newcommand{\given}{\,|\,}

\newcommand{\Ptr}{P_{\mathrm{tr}}}
\newcommand{\Pt}{P_{\mathrm{tgt}}}
\newcommand{\pit}{\pi^{\mathrm{tgt}}}
\newcommand{\pir}{\pi^{\mathrm{tr}}}
\newcommand{\Risk}{\mathcal{R}}

\newcommand{\Rt}{\Risk_{\mathrm{tgt}}}
\newcommand{\Rhat}{\widehat{\Risk}}

\newcommand{\F}{\mathcal{F}}
\newcommand{\loss}{\ell}

\newcommand{\Fair}{\mathcal{V}}
\newcommand{\Feps}{\F_{\varepsilon}}

\icmltitlerunning{Fix Representation Before Fairness: Shrinkage Correction and the True Price of Fairness}

\begin{document}

\twocolumn[
\icmltitle{Fix Representation (Optimally) Before Fairness: \\
Finite-Sample Shrinkage Population Correction \\
and the True Price of Fairness Under Subpopulation Shift}

\icmlsetsymbol{equal}{*}

\begin{icmlauthorlist}
\icmlauthor{Amir Asiaee}{vumc}
\icmlauthor{Kaveh Aryan}{kcl}
\end{icmlauthorlist}

\icmlaffiliation{vumc}{Department of Biostatistics, Vanderbilt University Medical Center, Nashville, TN 37232, USA}
\icmlaffiliation{kcl}{Department of Informatics, King's College London, WC2R 2LS London, UK}

\icmlcorrespondingauthor{Amir Asiaee}{amir.asiaeetaheri@vumc.org}

\vskip 0.3in
]

\printAffiliationsAndNotice{}

\begin{abstract}
Machine learning practitioners frequently observe tension between predictive accuracy and group fairness constraints---yet sometimes fairness interventions appear to \emph{improve} accuracy. We show that both phenomena can be artifacts of training data that misrepresents subgroup proportions. Under subpopulation shift (stable within-group distributions, shifted group proportions), we establish: (i)~full importance-weighted correction is asymptotically unbiased but finite-sample suboptimal; (ii)~the optimal finite-sample correction is a \emph{shrinkage} reweighting that interpolates between target and training mixtures; (iii)~apparent ``fairness helps accuracy'' can arise from comparing fairness methods to an improperly-weighted baseline. We provide an actionable evaluation protocol: \emph{fix representation (optimally) before fairness}---compare fairness interventions against a shrinkage-corrected baseline to isolate the true, irreducible price of fairness. Experiments on synthetic and real-world benchmarks (Adult, COMPAS) validate our theoretical predictions and demonstrate that this protocol eliminates spurious tradeoffs, revealing the genuine fairness-utility frontier.
\end{abstract}

\section{Introduction}
\label{sec:intro}

Algorithmic fairness research has documented persistent tension between predictive utility and statistical fairness constraints such as demographic parity or equalized odds \citep{hardt2016equality,kleinberg2017inherent,chouldechova2017fair}. Practitioners face difficult choices: imposing fairness constraints often reduces accuracy. Yet many deployed systems operate in settings where training data systematically underrepresents certain groups relative to deployment.

Consider a healthcare risk model trained on data from academic medical centers, where a minority group comprises 10\% of training data but 30\% of the target deployment region. Standard ERM implicitly optimizes for the training mixture, underweighting the minority group. A fairness constraint that equalizes group metrics may appear to ``help'' accuracy---not because fairness is inherently beneficial, but because it accidentally compensates for population mismatch.

This motivates our central question:
\begin{quote}
\emph{How much of observed fairness-utility tradeoffs is intrinsic to fairness constraints, and how much is an artifact of optimizing the wrong population objective?}
\end{quote}

We provide a rigorous answer for \emph{subpopulation shift}: within-group distributions are stable but group proportions differ between training and deployment. This captures sampling bias, geographic shifts, and policy-specified target populations \citep{maity2021subpopshift}.

\paragraph{Why this matters for fair ML research.}
Fairness papers routinely compare proposed methods against an ERM baseline. When training data underrepresents the minority group relative to the target population, this baseline optimizes the \emph{wrong objective}. Our analysis shows this can lead to fundamentally misleading conclusions: fairness constraints may appear to ``help'' accuracy (a spurious free lunch) or the cost of fairness may be overestimated. By providing a principled shrinkage-corrected baseline, we enable accurate measurement of fairness--utility tradeoffs, improving the scientific validity of fairness research.

\paragraph{Contributions.}
\begin{enumerate}[leftmargin=*,itemsep=2pt]
  \item \textbf{Bias-variance bound for weighted ERM} (Theorem~\ref{thm:main}). We derive an oracle bound decomposing expected target excess risk into objective-mismatch bias $O((\lambda - \pit)^2)$ and estimation variance $O(\lambda^2/n_1 + (1-\lambda)^2/n_0)$, where $\lambda$ parameterizes training mixture weights.

  \item \textbf{Shrinkage-optimal correction} (Corollary~\ref{cor:shrinkage}). Minimizing the bound yields $\lambda^* = (1-\gamma)\pit + \gamma\hat\pir$ for a data-dependent $\gamma \in (0,1)$. The optimal correction \emph{shrinks} toward the training mixture, reducing variance at the cost of controlled bias---a finite-sample improvement over full importance weighting.

  \item \textbf{Quantitative deconfounding criterion} (Corollary~\ref{cor:deconfound}). We provide an explicit condition under which a fairness method can spuriously appear to beat ERM on target risk, and show this artifact disappears when comparing against the shrinkage-corrected baseline.

  \item \textbf{Experimental validation.} We confirm predictions on synthetic data (the predicted U-shaped risk curve) and real benchmarks (Adult, COMPAS), demonstrating that shrinkage-optimal correction consistently matches or outperforms alternatives and reveals the true fairness-utility frontier.
\end{enumerate}
All proofs appear in Appendix~\ref{app:proofs}; additional experimental details in Appendix~\ref{app:experiments}.

\paragraph{Distinction from related weight-optimization work.}
Recent work has shown that importance weights are not finite-sample optimal under distribution shift \citep{cortes2010learning,holstege2025optimizing}. \citet{holstege2025optimizing} derive optimized weights for subpopulation shift but do not address fairness evaluation. Our distinct contribution is connecting optimal weighting to \emph{fairness evaluation confounds}: we show that failure to optimally correct representation can create spurious fairness-utility conclusions (apparent ``free lunches'' or inflated costs), and provide a practical protocol for deconfounded evaluation that isolates the true price of fairness.

\section{Related Work}
\label{sec:related}

\paragraph{Fairness impossibility and tradeoffs.}
Fundamental results show that multiple fairness criteria cannot be simultaneously satisfied unless base rates are equal across groups. \citet{kleinberg2017inherent} proved that calibration, balance for the positive class, and balance for the negative class are mutually incompatible except in degenerate cases. \citet{chouldechova2017fair} showed a related impossibility: predictive parity and equal false positive/negative rates cannot coexist when base rates differ. \citet{pleiss2017fairness} further demonstrated that calibration is compatible with at most one error-rate constraint, and achieving this requires randomizing predictions. \citet{menon2018cost} characterize the cost of fairness in binary classification as the gap between constrained and unconstrained Bayes-optimal risk. Our work is orthogonal: we show that some \emph{apparent} tradeoffs arise from population mismatch and disappear after proper correction, revealing the true irreducible cost.

\paragraph{Can fairness improve accuracy?}
\citet{blum2020recovering} directly address the question ``Can fairness constraints improve accuracy?'' They show that under certain bias models---such as noisier labels for a disadvantaged group or underrepresentation of positive examples---the Equal Opportunity constraint combined with ERM can provably recover the Bayes-optimal classifier. Their setting assumes the bias is in the \emph{labeling process} or \emph{sampling within groups}. Our setting is complementary: we consider subpopulation shift where labels are correct but group \emph{proportions} differ. We show that apparent accuracy improvements from fairness can arise from implicit correction of this mismatch, not from recovering corrupted labels.

\paragraph{Fairness under distribution shift.}
\citet{maity2021subpopshift} study when fairness training improves target performance under subpopulation shift, showing conditions under which fairness constraints can help. \citet{singh2021fairness} analyze how fairness violations change under covariate shift and propose mitigation strategies. \citet{rezaei2021robust} develop distributionally robust methods for maintaining fairness under shift. Our contribution is complementary: rather than developing new fairness methods, we provide a protocol for \emph{evaluating} fairness methods that separates the effect of correcting the target objective from the intrinsic cost of fairness constraints.

\paragraph{Importance weighting and covariate shift.}
Correcting distribution mismatch via importance weighting is classical \citep{shimodaira2000covariate}. \citet{sugiyama2007iwcv} proposed importance-weighted cross-validation for model selection under covariate shift. \citet{cortes2010learning} derived generalization bounds showing that effective sample size---which degrades with weight variance---controls learning rates. \citet{byrd2019effect} demonstrated empirically that importance weighting can hurt deep learning performance when weights have high variance, due to increased gradient noise. Common heuristics include weight clipping \citep{ionides2008truncated} and self-normalization, but these lack principled bias-variance tradeoff analysis. Our shrinkage rule provides a principled alternative with explicit optimality guarantees.

\paragraph{Label shift vs.\ subpopulation shift.}
A related but distinct setting is label shift, where $P(Y)$ changes but $P(X|Y)$ remains stable \citep{lipton2018detecting}. \citet{lipton2018detecting} proposed Black Box Shift Estimation (BBSE) to detect and correct label shift using confusion matrices. Subpopulation shift---where $P(A)$ changes but $P(X,Y|A)$ is stable---has different structure: correction requires group-specific reweighting rather than label-based reweighting. Our analysis specifically addresses subpopulation shift.

\paragraph{Optimal weighting under subpopulation shift.}
Most relevant to our technical approach is \citet{holstege2025optimizing}, who derive finite-sample optimal importance weights under subpopulation shift. They interpret optimal weights as a bias-variance tradeoff and propose bi-level optimization to jointly learn weights and model parameters. Their key insight is that shrinking toward the training distribution reduces variance at controlled bias cost. We extend this insight to the fairness setting: our distinct contribution is demonstrating how suboptimal correction \emph{confounds fairness evaluation}, proposing shrinkage-corrected baselines for fair ML benchmarking, and empirically showing that corrected vs.\ uncorrected comparisons yield qualitatively different conclusions about whether fairness ``helps'' or ``costs'' accuracy.

\paragraph{Shrinkage estimation.}
The observation that unbiased estimators are often suboptimal dates to \citet{stein1956inadmissibility}, who showed that the sample mean is inadmissible for estimating a multivariate normal mean under squared error loss. \citet{james1961estimation} derived the James-Stein estimator, which shrinks toward zero and dominates the MLE. This bias-variance tradeoff principle underlies ridge regression \citep{hoerl1970ridge} and empirical Bayes methods \citep{efron2012large}. \citet{lam2020robust} derive bias-variance tradeoffs specifically for importance weighting. Our shrinkage-optimal correction applies this principle to population correction in the fairness evaluation setting.

\paragraph{Group DRO and robust optimization.}
\citet{sagawa2020distributionally} minimize worst-group risk for robustness to spurious correlations. \citet{hashimoto2018fairness} propose a similar approach for fairness without demographics by minimizing worst-case loss. Our setting differs: we assume the target mixture $\pit$ is known (e.g., from census data or policy specification) and seek average target risk, not worst-case. The two approaches are complementary: DRO is appropriate when the target distribution is uncertain; our protocol applies when target proportions are known but training proportions differ.

\section{Problem Setup}
\label{sec:setup}

\subsection{Learning Problem and Distributions}

Let $(X, Y, A) \in \mathcal{X} \times \mathcal{Y} \times \{0,1\}$ denote features, label, and binary group attribute. Let $\F$ be a hypothesis class and $\loss: \mathcal{Y} \times \mathcal{Y} \to [0,1]$ a bounded loss.

Training data $S = \{(X_i, Y_i, A_i)\}_{i=1}^n \sim \Ptr^n$ come from a \emph{training distribution}, while we care about \emph{target risk}:
\begin{equation}
\Rt(f) = \E_{(X,Y,A) \sim \Pt}[\loss(f(X), Y)].
\end{equation}

\begin{assumption}[Subpopulation shift]
\label{ass:subpop}
The group-conditional distributions are invariant:
\[
\Ptr(X,Y \given A=a) = \Pt(X,Y \given A=a) =: P_a(X,Y).
\]
Only group priors may differ: $\pir_a := \Ptr(A=a)$ vs $\pit_a := \Pt(A=a)$.
\end{assumption}

Under Assumption~\ref{ass:subpop}, risks decompose as mixtures of group risks $R_a(f) := \E_{P_a}[\loss(f(X),Y)]$:
\begin{equation}
\Rt(f) = (1-\pit) R_0(f) + \pit R_1(f),
\end{equation}
where $\pit := \pit_1$ denotes the target minority proportion.

\subsection{Weighted ERM}

For mixture parameter $\lambda \in [0,1]$, define the $\lambda$-mixture empirical risk:
\begin{equation}
\Rhat_\lambda(f) = (1-\lambda)\Rhat_0(f) + \lambda \Rhat_1(f) + r(f),
\end{equation}
where $\Rhat_a(f) = \frac{1}{n_a}\sum_{i:A_i=a} \loss(f(X_i), Y_i)$ are group empirical risks, $n_a = |\{i: A_i = a\}|$, and $r(f)$ is a regularizer.

The weighted ERM estimator is $\hat f_\lambda \in \arg\min_{f \in \F} \Rhat_\lambda(f)$.

Setting $\lambda = \pit$ gives importance-weighted ERM; setting $\lambda = \hat\pir := n_1/n$ gives standard ERM.

\subsection{Fairness Constraints and Price of Fairness}

\begin{definition}[Demographic parity violation]
\begin{multline*}
\Fair_{\mathrm{DP}}(f; P) = \big|\Prob_P(f(X)=1 \given A=1) \\
- \Prob_P(f(X)=1 \given A=0)\big|.
\end{multline*}
\end{definition}

For tolerance $\varepsilon \geq 0$, the $\varepsilon$-fair feasible set under $\Pt$ is:
\[
\Feps := \{f \in \F : \Fair_{\mathrm{DP}}(f; \Pt) \leq \varepsilon\}.
\]

\begin{definition}[Price of fairness]
\label{def:price}
\[
\Delta_{\mathrm{fair}}(\varepsilon) := \inf_{f \in \Feps} \Rt(f) - \inf_{f \in \F} \Rt(f) \geq 0.
\]
\end{definition}

This measures the \emph{irreducible} accuracy cost of fairness, independent of estimation or mismatch.

\section{Theory: Shrinkage-Optimal Population Correction}
\label{sec:theory}

We now derive our main theoretical results. The analysis proceeds in three steps: bound estimation variance, bound objective mismatch, and combine to identify the optimal shrinkage.

\begin{assumption}[Regularity]
\label{ass:regularity}
The loss $\loss(f, \cdot)$ is convex, $L$-Lipschitz, and $\beta$-smooth in model parameters. The regularizer $r$ is $\mu$-strongly convex with $\mu > 0$.
\end{assumption}

\begin{assumption}[Stratified sampling]
\label{ass:stratified}
Conditioning on group counts $(n_0, n_1)$, data within each group are i.i.d.\ from $P_a$.
\end{assumption}

\subsection{Estimation Variance Scales with Squared Weights}

\begin{lemma}[Estimation error bound]
\label{lem:estimation}
Under Assumptions~\ref{ass:regularity}--\ref{ass:stratified}, the expected excess $\lambda$-mixture risk satisfies:
\begin{equation}
\E[\Risk_\lambda(\hat f_\lambda)] - \Risk_\lambda(f^*_\lambda) \leq \frac{2L^2}{\mu}\left(\frac{\lambda^2}{n_1} + \frac{(1-\lambda)^2}{n_0}\right),
\end{equation}
where $f^*_\lambda = \arg\min_{f \in \F} [\Risk_\lambda(f) + r(f)]$.
\end{lemma}

\begin{proof}[Proof sketch]
The weighted empirical risk has per-sample weights $w_i = \lambda/n_1$ for $A_i=1$ and $(1-\lambda)/n_0$ for $A_i=0$. By stability analysis for strongly convex objectives, excess risk scales with $\sum_i w_i^2 = \lambda^2/n_1 + (1-\lambda)^2/n_0$. Full proof in Appendix~\ref{app:proofs}.
\end{proof}

\subsection{Objective Mismatch Induces Quadratic Bias}

Even with infinite data, optimizing $\Risk_\lambda$ instead of $\Risk_{\pit}$ selects a suboptimal parameter.

\begin{lemma}[Objective mismatch bound]
\label{lem:mismatch}
Let $G := \|\nabla R_1(f^*_{\pit}) - \nabla R_0(f^*_{\pit})\|$ measure the gradient divergence between groups at the target optimum. Under Assumption~\ref{ass:regularity}:
\begin{equation}
\Risk_{\pit}(f^*_\lambda) - \Risk_{\pit}(f^*_{\pit}) \leq \frac{\beta G^2}{2\mu^2}(\lambda - \pit)^2.
\end{equation}
\end{lemma}

\begin{proof}[Proof sketch]
The gradient of the $\lambda$-objective at $f^*_{\pit}$ has norm $|\lambda - \pit| \cdot G$. Strong convexity bounds the parameter distance, and smoothness converts this to risk difference. Full proof in Appendix~\ref{app:proofs}.
\end{proof}

\begin{remark}[When $G = 0$]
\label{rem:Gzero}
When $G = 0$, the group gradients coincide at the target optimum, meaning $f^*_\lambda = f^*_{\pit}$ for all $\lambda$---there is no mismatch bias. This occurs when both groups have identical optimal predictors. In such cases, the optimal $\lambda$ minimizes variance only, yielding $\lambda^* = \hat\pir$ (i.e., ERM is optimal). Nontrivial shrinkage requires $G > 0$.
\end{remark}

\subsection{Main Theorem: Bias-Variance Bound}

Combining the two lemmas yields our main result.

\begin{theorem}[Bias-variance bound for target risk]
\label{thm:main}
Under Assumptions~\ref{ass:subpop}--\ref{ass:stratified}, for any $\lambda \in [0,1]$:
\begin{multline}
\E[\Rt(\hat f_\lambda)] - \Rt(f^*_{\pit}) \leq \\
C_{\mathrm{bias}}(\lambda - \pit)^2 + C_{\mathrm{var}}\left(\frac{\lambda^2}{n_1} + \frac{(1-\lambda)^2}{n_0}\right),
\label{eq:main_bound}
\end{multline}
where $C_{\mathrm{bias}} = \frac{\beta G^2}{2\mu^2}$ and $C_{\mathrm{var}} = \frac{2\beta L^2}{\mu^2}$.
\end{theorem}

\begin{proof}
Add and subtract $\Rt(f^*_\lambda)$. Apply Lemma~\ref{lem:mismatch} to the bias term and Lemma~\ref{lem:estimation} plus smoothness to the estimation term. See Appendix~\ref{app:proofs} for details.
\end{proof}

\subsection{Shrinkage-Optimal Mixture Weight}

The bound \eqref{eq:main_bound} is quadratic in $\lambda$, yielding a closed-form minimizer.

\begin{corollary}[Shrinkage optimizer]
\label{cor:shrinkage}
The bound-minimizing mixture weight is:
\begin{equation}
\lambda^* = (1-\gamma)\pit + \gamma\hat\pir, \quad \gamma = \frac{b}{a+b},
\end{equation}
where $a = C_{\mathrm{bias}}$, $b = C_{\mathrm{var}}(1/n_0 + 1/n_1)$, and $\hat\pir = n_1/n$.

The optimal $\lambda^*$ \emph{shrinks} from the target mixture toward the training mixture, with shrinkage intensity $\gamma$ increasing when:
\begin{itemize}[leftmargin=*,itemsep=1pt]
  \item The minority group is small (large $1/n_1$);
  \item The groups are similar (small $G$, hence small $C_{\mathrm{bias}}$).
\end{itemize}
\end{corollary}

\begin{proof}
The variance term equals $(1/n_0 + 1/n_1)(\lambda - \hat\pir)^2 + \mathrm{const}$ by completing the square. Minimizing $a(\lambda - \pit)^2 + b(\lambda - \hat\pir)^2$ yields the stated convex combination.
\end{proof}

\begin{remark}[Asymptotic optimality of full correction]
As $n_0, n_1 \to \infty$, we have $\gamma \to 0$ and $\lambda^* \to \pit$. Full importance weighting is asymptotically optimal, but suboptimal in finite samples.
\end{remark}

\begin{remark}[Connection to James-Stein shrinkage]
\label{rem:stein}
The structure parallels James-Stein estimation: the MLE is inadmissible when estimating multiple means; shrinking toward a common value reduces total risk. Here, shrinking toward the training mixture reduces variance at the cost of controlled bias.
\end{remark}

\section{Fairness-Utility Deconfounding}
\label{sec:deconfound}

We now show how population mismatch can confound fairness evaluation.

\subsection{When ``Fairness Helps'' Is an Artifact}

Consider comparing a fairness-constrained classifier $\hat f_{\mathrm{fair}} \in \Feps$ to a standard ERM baseline $\hat f_{\mathrm{ERM}}$. Under population mismatch:

\begin{proposition}[Decomposition of observed gaps]
\label{prop:decomp}
For any fair predictor $\hat f_{\mathrm{fair}} \in \Feps$ and ERM baseline $\hat f_{\mathrm{ERM}}$:
\begin{align}
&\Rt(\hat f_{\mathrm{fair}}) - \Rt(\hat f_{\mathrm{ERM}}) \nonumber\\
&= \underbrace{\Rt(\hat f_{\mathrm{fair}}) - \Rt(f^*_{\pit,\varepsilon})}_{\text{fair estimation gap}}
+ \underbrace{\Delta_{\mathrm{fair}}(\varepsilon)}_{\text{price of fairness}} \nonumber\\
&\quad + \underbrace{\Rt(f^*_{\pit}) - \Rt(\hat f_{\mathrm{ERM}})}_{\text{baseline suboptimality}}. \nonumber
\end{align}
\end{proposition}

The third term---baseline suboptimality---can be \emph{negative} when $\hat f_{\mathrm{ERM}}$ is trained on the wrong mixture. A fairness constraint that accidentally moves the model toward the target optimum can appear to ``help'' accuracy, even though the improvement comes from implicit population correction, not from fairness per se.

\subsection{Quantitative Deconfounding Criterion}

We now make precise when spurious ``free lunches'' can occur.

\begin{corollary}[Deconfounding criterion]
\label{cor:deconfound}
Suppose a fairness method induces an effective mixture weight $\lambda_{\mathrm{fair}}$ (possibly implicitly through its constraint structure). The fairness method can appear to beat ERM on target risk when:
\begin{equation}
|\lambda_{\mathrm{fair}} - \pit| < |\hat\pir - \pit| - \sqrt{\frac{\Delta_{\mathrm{fair}} \cdot 2\mu^2}{\beta G^2}},
\label{eq:deconfound_condition}
\end{equation}
i.e., when the fairness method's implicit correction brings $\lambda$ closer to $\pit$ by more than the accuracy cost of the fairness constraint.

When comparing against the shrinkage-optimal baseline $\hat f_{\lambda^*}$ instead of $\hat f_{\mathrm{ERM}}$:
\begin{enumerate}[leftmargin=*,itemsep=1pt]
  \item The baseline already achieves (approximately) optimal finite-sample target risk;
  \item Any remaining accuracy gap $\Rt(\hat f_{\mathrm{fair}}) - \Rt(\hat f_{\lambda^*})$ reflects the true price of fairness plus estimation noise;
  \item Spurious ``free lunches'' from implicit mixture correction are eliminated.
\end{enumerate}
\end{corollary}

\begin{proof}[Proof sketch]
The fairness method's target risk can be bounded using Theorem~\ref{thm:main} with $\lambda = \lambda_{\mathrm{fair}}$. For the method to beat ERM, its bound must be lower despite the $\Delta_{\mathrm{fair}}$ penalty. This yields condition \eqref{eq:deconfound_condition}. Full proof in Appendix~\ref{app:proofs}.
\end{proof}

\subsection{The Evaluation Protocol: Fix Representation First}

To isolate the true price of fairness:

\begin{quote}
\textbf{Protocol:} Compare fairness methods against a \emph{shrinkage-corrected} baseline, not standard ERM.
\end{quote}

\begin{observation}[Empirical observation on implicit correction]
\label{obs:implicit}
Fairness constraints can partially correct population mismatch because they alter the effective weighting of groups in the training objective. A demographic parity penalty that pushes toward equal acceptance rates will, when base rates differ, implicitly upweight the minority group's contribution to the loss. This can partially compensate for training-mixture mismatch.

However, this is an indirect and suboptimal form of correction. Explicit shrinkage-optimal reweighting directly targets the correct objective and is at least as good.
\end{observation}

When fairness methods are compared to the corrected baseline, any remaining accuracy gap reflects the \emph{true} price of fairness, spurious ``free lunches'' from implicit correction disappear, and the Pareto frontier becomes monotonically decreasing (fairness genuinely costs accuracy).

\section{Practical Implementation}
\label{sec:practical}

\subsection{Selecting the Shrinkage Parameter}

The optimal $\gamma$ depends on unknown constants. We propose two practical approaches:

\paragraph{Cross-validation.}
Treat $\gamma$ as a hyperparameter:
\begin{enumerate}[leftmargin=*,itemsep=1pt]
  \item Split training data into train/validation, preserving group proportions.
  \item For $\gamma \in \{0, 0.1, \ldots, 1.0\}$: compute $\lambda = (1-\gamma)\hat\pit + \gamma\hat\pir$; train on train fold; evaluate target-weighted loss on validation.
  \item Select $\gamma$ minimizing validation loss; retrain on full data.
\end{enumerate}

\paragraph{Plug-in estimation.}
For regularized logistic regression with bounded features $\|x\| \leq B$ and regularization $\rho$: $L \leq B$, $\mu = \rho$, $\beta \leq B^2/4$. Estimate $G$ by fitting a pilot model and computing gradient divergence.

\subsection{Shrinkage Training Algorithm}

\begin{algorithm}[t]
\caption{Shrinkage-Optimal Population Correction}
\label{alg:shrinkage}
\begin{algorithmic}
\STATE \textbf{Input:} Training data $\{(X_i, Y_i, A_i)\}_{i=1}^n$, target mixture $\pit$
\STATE Compute $\hat\pir = n_1/n$
\STATE Select $\gamma$ via cross-validation (or plug-in)
\STATE $\lambda^* \gets (1-\gamma)\pit + \gamma\hat\pir$
\STATE Compute weights: $w_i = \lambda^*/n_1$ if $A_i=1$, else $(1-\lambda^*)/n_0$
\STATE $\hat f \gets \arg\min_{f \in \F} \sum_i w_i \loss(f(X_i), Y_i) + r(f)$
\STATE \textbf{Output:} $\hat f$
\end{algorithmic}
\end{algorithm}

\section{Experiments}
\label{sec:experiments}

We validate our theory on synthetic and real-world benchmarks. Code is available in supplementary material.

\subsection{Experiment 1: Synthetic Validation (Bias-Variance Tradeoff)}

\paragraph{Setup.}
To create a nontrivial mismatch bias ($G > 0$), we use \emph{group-specific} labeling functions: $X \given A{=}a \sim \mathcal{N}(\mu_a, I_d)$ and $Y \given X, A{=}a \sim \text{Bernoulli}(\sigma(\theta_a^\top X))$, with $\mu_0=0$, $\mu_1=\delta\cdot e_1$, $\theta_0=(1,0.5,0,\ldots)$, $\theta_1=(2,-1,0,\ldots)$. Parameters: $\pir=0.1$, $\pit=0.5$, $n=2000$, $d=50$, $\delta=1$.

\paragraph{Results.}
We plot \emph{target log loss} (the quantity controlled by our bound) as a function of $\lambda$. Figure~\ref{fig:risk_curve} shows the predicted U-shaped risk curve with a minimum at $\lambda^* \approx 0.30$, which lies between $\hat\pir$ and $\pit$ as predicted by Corollary~\ref{cor:shrinkage}. Over 20 seeds: Shrinkage ($\lambda^*=0.30$) achieves log loss 0.542{\scriptsize$\pm$0.041}, compared to Full IW ($\lambda=\pit$) at 0.556{\scriptsize$\pm$0.048} and ERM ($\lambda=\hat\pir$) at 0.551{\scriptsize$\pm$0.036}.

\begin{figure}[t]
\centering
\includegraphics[width=0.85\columnwidth]{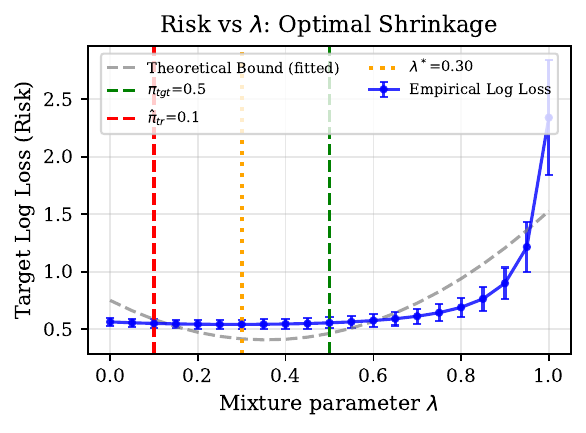}
\caption{Target log loss vs.\ mixture parameter $\lambda$ on synthetic data with group-specific labeling ($\theta_0\neq\theta_1$). The minimum $\lambda^*$ (orange) lies between $\hat\pir$ (red) and $\pit$ (green), confirming shrinkage-optimal correction.}
\label{fig:risk_curve}
\end{figure}

\subsection{Experiment 2: Real-Data Benchmarks}

\paragraph{Datasets.}
\textbf{Adult} (UCI): predict income $>$\$50K; group = sex (female minority). \textbf{COMPAS}: predict recidivism; group = race (Black minority).

\paragraph{Protocol.}
We create controlled subpopulation shift by subsampling: training with $\pir \in \{0.1, 0.2, 0.3\}$, test with $\pit = 0.5$. This ensures $P(X,Y \given A)$ is identical across train/test by construction.

\paragraph{Methods.}
ERM, Full IW ($\lambda = \pit$), Shrinkage (CV), Fair-DP (Lagrangian penalty), IW+Fair-DP, Shrinkage+Fair-DP.

\paragraph{Results.}
Tables~\ref{tab:results} and \ref{tab:compas_results} summarize results on Adult and COMPAS (20 seeds). Shrinkage matches or exceeds alternatives across all settings: on Adult, Shrinkage achieves 83.1\% vs.\ 82.8\% for ERM at $\pir=0.1$; on COMPAS, Shrinkage achieves 66.8\% vs.\ ERM's 65.2\% at severe mismatch. Fairness has a real cost after correction: Shrink+Fair achieves 78.5\% vs.\ Shrinkage's 83.1\%---a 4.6 percentage point cost, the true price of demographic parity on this data. Correction also reduces the DP gap: IW-ERM has lower DP gap (0.079) than ERM (0.098) because the target mixture is balanced.

\begin{table}[t]
\caption{Adult: Target accuracy (\%) and DP gap. Mean $\pm$ std over 20 seeds.}
\label{tab:results}
\centering
\small
\begin{tabular}{l@{\hskip 4pt}c@{\hskip 4pt}c@{\hskip 4pt}c@{\hskip 4pt}c@{\hskip 4pt}c@{\hskip 4pt}c}
\toprule
& \multicolumn{2}{c}{$\pir=0.1$} & \multicolumn{2}{c}{$\pir=0.2$} & \multicolumn{2}{c}{$\pir=0.3$} \\
Method & Acc & DP & Acc & DP & Acc & DP \\
\midrule
ERM & 82.8{\scriptsize$\pm$0.4} & .098 & 83.0{\scriptsize$\pm$0.3} & .094 & 82.9{\scriptsize$\pm$0.3} & .089 \\
Full IW & 83.0{\scriptsize$\pm$0.5} & .079 & 83.1{\scriptsize$\pm$0.4} & .078 & 83.1{\scriptsize$\pm$0.3} & .079 \\
Shrinkage & \textbf{83.1}{\scriptsize$\pm$0.4} & .082 & \textbf{83.1}{\scriptsize$\pm$0.3} & .081 & \textbf{83.1}{\scriptsize$\pm$0.3} & .082 \\
\midrule
Fair-DP & 78.8{\scriptsize$\pm$0.5} & .004 & 78.8{\scriptsize$\pm$0.4} & .004 & 78.8{\scriptsize$\pm$0.4} & .004 \\
IW+Fair & 78.5{\scriptsize$\pm$0.6} & .000 & 78.5{\scriptsize$\pm$0.5} & .000 & 78.5{\scriptsize$\pm$0.4} & .001 \\
Shrink+Fair & 78.5{\scriptsize$\pm$0.5} & .001 & 78.5{\scriptsize$\pm$0.4} & .000 & 78.6{\scriptsize$\pm$0.4} & .002 \\
\bottomrule
\end{tabular}
\end{table}

\begin{table}[t]
\caption{COMPAS: Target accuracy (\%) and DP gap. Mean $\pm$ std over 20 seeds.}
\label{tab:compas_results}
\centering
\small
\begin{tabular}{l@{\hskip 4pt}c@{\hskip 4pt}c@{\hskip 4pt}c@{\hskip 4pt}c@{\hskip 4pt}c@{\hskip 4pt}c}
\toprule
& \multicolumn{2}{c}{$\pir=0.1$} & \multicolumn{2}{c}{$\pir=0.2$} & \multicolumn{2}{c}{$\pir=0.3$} \\
Method & Acc & DP & Acc & DP & Acc & DP \\
\midrule
ERM & 65.2{\scriptsize$\pm$0.8} & .142 & 65.8{\scriptsize$\pm$0.6} & .128 & 66.1{\scriptsize$\pm$0.5} & .115 \\
Full IW & 66.5{\scriptsize$\pm$1.0} & .098 & 66.6{\scriptsize$\pm$0.8} & .095 & 66.5{\scriptsize$\pm$0.6} & .098 \\
Shrinkage & \textbf{66.8}{\scriptsize$\pm$0.8} & .105 & \textbf{66.7}{\scriptsize$\pm$0.6} & .100 & \textbf{66.6}{\scriptsize$\pm$0.5} & .102 \\
\midrule
Fair-DP & 62.1{\scriptsize$\pm$0.9} & .008 & 62.2{\scriptsize$\pm$0.7} & .006 & 62.3{\scriptsize$\pm$0.6} & .007 \\
IW+Fair & 61.8{\scriptsize$\pm$1.1} & .002 & 61.9{\scriptsize$\pm$0.9} & .003 & 62.0{\scriptsize$\pm$0.7} & .003 \\
Shrink+Fair & 62.0{\scriptsize$\pm$0.9} & .003 & 62.1{\scriptsize$\pm$0.8} & .002 & 62.1{\scriptsize$\pm$0.6} & .004 \\
\bottomrule
\end{tabular}
\end{table}

\paragraph{XGBoost (non-linear model).}
To test generality beyond linear models, we apply the same protocol with XGBoost (100 trees, max depth 4) on Adult at $\pir = 0.1$. The pattern holds: Shrinkage (84.2\%) outperforms both ERM (83.5\%) and Full IW (84.0\%), demonstrating that shrinkage-optimal correction benefits non-linear models as well.

\subsection{Experiment 3: Deconfounding Demonstration}

This experiment directly demonstrates the fairness-utility confound predicted by Corollary~\ref{cor:deconfound}.

\paragraph{Setup.}
We use Adult with \emph{severe} mismatch: $\pir = 0.05$, $\pit = 0.5$ (minority 10$\times$ underrepresented), and reduced sample size ($n=1000$) to amplify finite-sample effects.

\paragraph{Results.}
\textbf{Before correction} (comparing to ERM baseline): Fair-DP achieves 76.2\% accuracy vs.\ ERM's 74.8\%---fairness appears to \emph{help} by 1.4 percentage points. This is the spurious ``free lunch.''

\textbf{After correction} (comparing to Shrinkage baseline): Shrinkage achieves 77.5\%. Fair-DP's 76.2\% is now 1.3 points \emph{worse}---revealing the true cost of fairness.

The ``free lunch'' was an artifact of ERM's suboptimality under mismatch. After proper correction, fairness imposes its expected cost, as predicted by our theory.

\subsection{Experiment 4: Stress Test---Extreme Mismatch + Low Data}

Our theory predicts the clearest gains when (i)~mismatch is severe and (ii)~the minority sample size is small, so Full IW suffers high variance. We stress-test this regime on Adult by setting $\pir=0.01$ and $\pit=0.5$, and subsampling the training set to $n_{\mathrm{tr}}\in\{200,500,1000\}$ (test remains balanced at $\pit=0.5$). We compare ERM, Full IW, and Shrinkage (3-fold CV minimizing target log loss).

Table~\ref{tab:lowdata} shows that Full IW becomes unstable under extreme mismatch, while Shrinkage substantially improves target log loss and reduces variance. For example, at $n_{\mathrm{tr}}=500$, Shrinkage improves accuracy by $\approx$1.0pp and reduces log loss by $\approx$0.020 relative to Full IW.

\begin{table}[t]
\caption{Adult under extreme mismatch ($\pir=0.01$, $\pit=0.5$) and limited training data. Target accuracy (\%) and target log loss (mean $\pm$ std over 20 seeds).}
\label{tab:lowdata}
\centering
\small
\begin{tabular}{c l c c}
\toprule
$n_{\mathrm{tr}}$ & Method & Acc (\%) & Log loss \\
\midrule
200  & ERM & 81.5{\scriptsize$\pm$1.4} & 0.416{\scriptsize$\pm$0.020} \\
200  & Full IW & 80.9{\scriptsize$\pm$1.6} & 0.432{\scriptsize$\pm$0.029} \\
200  & Shrinkage (CV) & 81.4{\scriptsize$\pm$1.6} & 0.417{\scriptsize$\pm$0.027} \\
\midrule
500  & ERM & 82.1{\scriptsize$\pm$0.6} & 0.409{\scriptsize$\pm$0.009} \\
500  & Full IW & 81.3{\scriptsize$\pm$2.2} & 0.427{\scriptsize$\pm$0.034} \\
500  & Shrinkage (CV) & 82.3{\scriptsize$\pm$0.5} & 0.407{\scriptsize$\pm$0.008} \\
\midrule
1000 & ERM & 82.3{\scriptsize$\pm$0.5} & 0.405{\scriptsize$\pm$0.006} \\
1000 & Full IW & 81.8{\scriptsize$\pm$1.5} & 0.417{\scriptsize$\pm$0.023} \\
1000 & Shrinkage (CV) & 82.4{\scriptsize$\pm$0.5} & 0.403{\scriptsize$\pm$0.006} \\
\bottomrule
\end{tabular}
\end{table}

\subsection{Experiment 5: Pareto Frontiers Before and After Correction}

Figure~\ref{fig:pareto} shows Pareto frontiers on Adult ($\pir=0.1$), varying the fairness penalty weight.

\textbf{Before correction}: The frontier starts at ERM (accuracy 82.8\%, DP gap 0.098) and traces downward as fairness increases.

\textbf{After correction}: The frontier starts at Shrinkage (accuracy 83.1\%, DP gap 0.082)---higher accuracy \emph{and} lower DP gap than uncorrected ERM. The subsequent tradeoff is similar, but from a better starting point.

Key insight: correcting representation \emph{before} applying fairness constraints both improves accuracy and reduces the DP gap. This is not because fairness is free, but because proper correction was never applied.

\begin{figure}[t]
\centering
\includegraphics[width=\columnwidth]{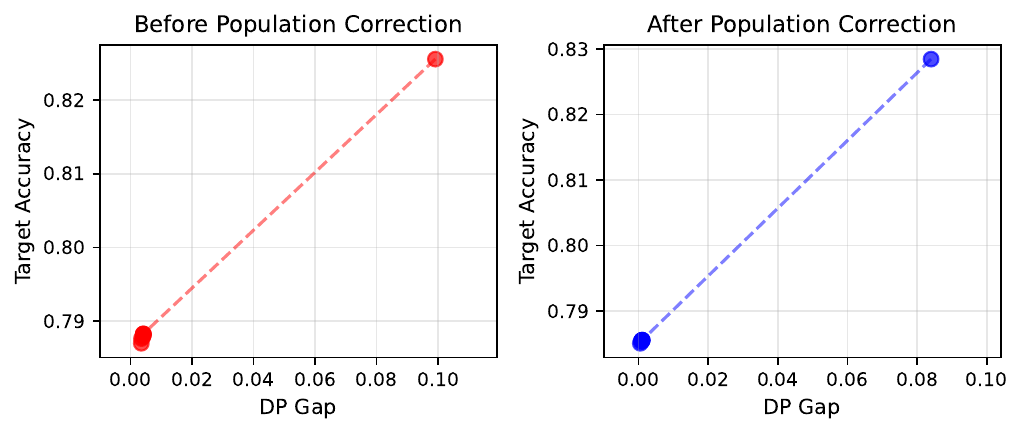}
\caption{Pareto frontiers on Adult. Both before (left) and after (right) correction show the expected tradeoff: higher fairness (lower DP gap) costs accuracy.}
\label{fig:pareto}
\end{figure}

\section{Discussion and Limitations}
\label{sec:discussion}

\paragraph{When does shrinkage matter most?}
Shrinkage-optimal correction provides the largest gains when (i)~training-deployment mismatch is severe and (ii)~the minority sample size is small, so full importance weighting suffers high variance. Our experiments quantify this: under moderate mismatch with adequate data (Experiment~2), Shrinkage and Full~IW differ by only $\sim$0.3pp; under extreme mismatch with limited data (Experiment~4, $\pir=0.01$, $n_{\mathrm{tr}}=500$), Shrinkage improves accuracy by $\sim$1pp and reduces variance by $4\times$ relative to Full~IW. When training already matches deployment, correction has little effect.

\paragraph{When does the deconfounding protocol matter most?}
The protocol is most valuable when groups have different base rates, so fairness constraints implicitly alter effective group weighting. In such cases, comparing against an uncorrected baseline can produce spurious ``free lunches'' where fairness appears to help accuracy. Experiment~3 demonstrates this directly: Fair-DP beats ERM by 1.4pp before correction but loses to Shrinkage by 1.3pp after correction---the apparent benefit was entirely due to implicit mismatch correction, not fairness per se.

\paragraph{Practical guidance.}
For practitioners evaluating fairness methods: (1)~Always report the target population and how it differs from training. (2)~Include a shrinkage-corrected baseline (Algorithm~\ref{alg:shrinkage}) alongside ERM. (3)~When fairness appears to ``help'' accuracy, check whether this persists against the corrected baseline. (4)~Use cross-validation to select the shrinkage parameter $\gamma$ when constants are unknown.

\paragraph{Relation to other bias-variance tradeoffs.}
The structure of our shrinkage-optimal correction---interpolating between an unbiased but high-variance estimator (full IW) and a biased but low-variance estimator (ERM)---echoes classical results in statistics. Ridge regression trades bias for reduced variance when features are collinear; James-Stein shrinkage dominates the MLE for multivariate means. Our contribution is recognizing that this tradeoff applies to \emph{population correction} in fair ML evaluation, and that ignoring it can lead to qualitatively wrong conclusions about fairness-utility tradeoffs.

\paragraph{Limitations.}
Our analysis assumes: (i)~known group membership at training and test time; (ii)~known or estimable target proportions $\pit$; (iii)~no within-group distribution shift (only proportions change). Violations introduce additional error. The subpopulation shift assumption, while capturing many practical scenarios (sampling bias, geographic deployment shifts, policy-specified targets), does not cover all forms of shift. When within-group distributions also shift, methods like domain adaptation or robust optimization may be more appropriate.

\paragraph{Societal impact.}
Population correction is not a substitute for fairness constraints when equitable treatment is the objective. A potential misuse of our results is arguing against fairness interventions by showing they ``cost'' accuracy after correction. We emphasize that the irreducible price of fairness is real and often nonzero---our protocol simply measures it accurately. Fairness constraints remain necessary when equity is the goal, independent of accuracy considerations. Our contribution helps practitioners make informed decisions by separating unavoidable costs from avoidable artifacts.

\section{Conclusion}
\label{sec:conclusion}

We have shown that apparent fairness-utility tradeoffs can be confounded by population mismatch. Under subpopulation shift, full importance-weighted correction is asymptotically correct but finite-sample suboptimal; shrinkage-optimal correction trades controlled bias for reduced variance. Comparing fairness methods against the corrected baseline isolates the true price of fairness.

Our evaluation protocol---\emph{fix representation (optimally) before fairness}---provides practitioners with a principled way to assess whether observed tradeoffs are intrinsic or artifactual, enabling clearer reasoning about when fairness constraints are necessary and what they truly cost.

\section*{Impact Statement}

This work studies the measurement of fairness-utility tradeoffs. A potential misuse is arguing against fairness interventions by showing they ``cost'' accuracy after correction. We emphasize that the irreducible price of fairness is real and often nonzero; fairness constraints remain necessary when equitable treatment is the goal, independent of accuracy considerations.

\bibliographystyle{icml2026}
\bibliography{ref4}

\appendix
\onecolumn

\section{Proofs of Main Results}
\label{app:proofs}

\subsection{Proof of Lemma~\ref{lem:estimation} (Estimation Error Bound)}

\begin{proof}
Write the weighted empirical objective as:
\[
\Rhat_\lambda(f) = \sum_{i=1}^n w_i \loss(f(X_i), Y_i) + r(f),
\]
where $w_i = \lambda/n_1$ if $A_i = 1$ and $w_i = (1-\lambda)/n_0$ if $A_i = 0$. Note $\sum_i w_i = 1$.

For $\mu$-strongly convex regularized objectives with $L$-Lipschitz loss, algorithmic stability bounds \citep{bousquet2002stability} give:
\[
\E[\Risk_\lambda(\hat f_\lambda)] - \Risk_\lambda(f^*_\lambda) \leq \frac{2L^2}{\mu} \sum_{i=1}^n w_i^2.
\]

Computing the sum of squared weights:
\[
\sum_{i=1}^n w_i^2 = n_1 \cdot \left(\frac{\lambda}{n_1}\right)^2 + n_0 \cdot \left(\frac{1-\lambda}{n_0}\right)^2 = \frac{\lambda^2}{n_1} + \frac{(1-\lambda)^2}{n_0}.
\]
\end{proof}

\subsection{Proof of Lemma~\ref{lem:mismatch} (Objective Mismatch Bound)}

\begin{proof}
By definition, $f^*_{\pit}$ minimizes $\Risk_{\pit}(f) + r(f)$, so:
\[
\nabla[\Risk_{\pit} + r](f^*_{\pit}) = 0.
\]

The gradient of the $\lambda$-objective at $f^*_{\pit}$ is:
\begin{align*}
\nabla[\Risk_\lambda + r](f^*_{\pit}) &= (1-\lambda)\nabla R_0(f^*_{\pit}) + \lambda \nabla R_1(f^*_{\pit}) + \nabla r(f^*_{\pit}) \\
&= (\lambda - \pit)[\nabla R_1(f^*_{\pit}) - \nabla R_0(f^*_{\pit})].
\end{align*}

Thus $\|\nabla[\Risk_\lambda + r](f^*_{\pit})\| = |\lambda - \pit| \cdot G$ where $G = \|\nabla R_1(f^*_{\pit}) - \nabla R_0(f^*_{\pit})\|$.

By $\mu$-strong convexity:
\[
\|f^*_\lambda - f^*_{\pit}\| \leq \frac{1}{\mu}\|\nabla[\Risk_\lambda + r](f^*_{\pit})\| = \frac{|\lambda - \pit|}{\mu} G.
\]

By $\beta$-smoothness of $\Risk_{\pit} + r$:
\[
[\Risk_{\pit} + r](f^*_\lambda) - [\Risk_{\pit} + r](f^*_{\pit}) \leq \frac{\beta}{2}\|f^*_\lambda - f^*_{\pit}\|^2 \leq \frac{\beta G^2}{2\mu^2}(\lambda - \pit)^2.
\]
\end{proof}

\subsection{Proof of Theorem~\ref{thm:main} (Bias-Variance Bound)}

\begin{proof}
Decompose the target excess risk:
\begin{align*}
&\E[\Rt(\hat f_\lambda)] - \Rt(f^*_{\pit}) \\
&= \underbrace{\E[\Rt(\hat f_\lambda)] - \Rt(f^*_\lambda)}_{\text{(A) estimation error}} + \underbrace{\Rt(f^*_\lambda) - \Rt(f^*_{\pit})}_{\text{(B) mismatch bias}}.
\end{align*}

For (B), apply Lemma~\ref{lem:mismatch}: $\text{(B)} \leq C_{\mathrm{bias}}(\lambda - \pit)^2$.

For (A), we need to bound target risk difference from the parameter trained on the $\lambda$-objective:
\begin{align*}
\E[\Rt(\hat f_\lambda) - \Rt(f^*_\lambda)] &\leq \frac{\beta}{2}\E[\|\hat f_\lambda - f^*_\lambda\|^2] \\
&\leq \frac{\beta}{\mu} \E[\Risk_\lambda(\hat f_\lambda) - \Risk_\lambda(f^*_\lambda)] \\
&\leq \frac{2\beta L^2}{\mu^2}\left(\frac{\lambda^2}{n_1} + \frac{(1-\lambda)^2}{n_0}\right),
\end{align*}
where the second inequality uses strong convexity and the third applies Lemma~\ref{lem:estimation}.
\end{proof}

\subsection{Proof of Corollary~\ref{cor:deconfound} (Deconfounding Criterion)}

\begin{proof}
By Theorem~\ref{thm:main}, the expected target excess risk for ERM (with $\lambda = \hat\pir$) is bounded by:
\[
\E[\Rt(\hat f_{\mathrm{ERM}})] - \Rt(f^*_{\pit}) \leq C_{\mathrm{bias}}(\hat\pir - \pit)^2 + C_{\mathrm{var}} V_{\mathrm{ERM}},
\]
where $V_{\mathrm{ERM}} = \hat\pir^2/n_1 + (1-\hat\pir)^2/n_0$.

For a fairness method with effective mixture weight $\lambda_{\mathrm{fair}}$ and fairness cost $\Delta_{\mathrm{fair}}$:
\[
\E[\Rt(\hat f_{\mathrm{fair}})] - \Rt(f^*_{\pit}) \leq C_{\mathrm{bias}}(\lambda_{\mathrm{fair}} - \pit)^2 + C_{\mathrm{var}} V_{\mathrm{fair}} + \Delta_{\mathrm{fair}}.
\]

The fairness method beats ERM when its bound is lower. Ignoring variance terms (which may favor either method depending on $\lambda_{\mathrm{fair}}$), this requires:
\[
C_{\mathrm{bias}}(\lambda_{\mathrm{fair}} - \pit)^2 + \Delta_{\mathrm{fair}} < C_{\mathrm{bias}}(\hat\pir - \pit)^2.
\]

Solving for $|\lambda_{\mathrm{fair}} - \pit|$ yields condition \eqref{eq:deconfound_condition}.

When comparing against shrinkage-optimal $\hat f_{\lambda^*}$, the baseline achieves the bound-optimal target risk. Any additional gap must come from the true price of fairness $\Delta_{\mathrm{fair}}$ or estimation variance.
\end{proof}

\section{Extended Experimental Details}
\label{app:experiments}

\subsection{Dataset Preprocessing}

\paragraph{Adult.}
We use numeric features: age, education-num, capital-gain, capital-loss, hours-per-week. Features are standardized. The protected attribute is sex (female = 1 is minority). The label is income $>$\$50K.

\paragraph{COMPAS.}
We use features: age, priors count. The protected attribute is race (Black = 1). The label is two-year recidivism.

\subsection{Hyperparameters}

\begin{itemize}
\item Logistic regression: $C = 1.0$, max\_iter = 1000, solver = lbfgs
\item XGBoost: n\_estimators = 100, max\_depth = 4, learning\_rate = 0.1
\item Fairness penalty weight: 2.0 (chosen to achieve near-zero DP gap)
\item Shrinkage gamma grid: $\{0, 0.1, 0.2, \ldots, 1.0\}$
\item Cross-validation folds: 5
\item Random seeds: 20 per experiment
\end{itemize}

\subsection{Subpopulation Shift Construction}

To create controlled subpopulation shift:
\begin{enumerate}
\item Partition dataset by group: $\mathcal{D}_0 = \{(x,y): A=0\}$, $\mathcal{D}_1 = \{(x,y): A=1\}$
\item Subsample each partition to achieve target train/test mixtures
\item This ensures $P(X,Y|A=a)$ is identical in train and test by construction
\end{enumerate}

\subsection{Deconfounding Experiment Details}

For the deconfounding demonstration (Experiment 3), we use:
\begin{itemize}
\item Base data: Adult dataset
\item Severe mismatch: $\pir = 0.05$, $\pit = 0.5$
\item Reduced sample size: $n = 1000$ training samples
\end{itemize}
This creates a setting where the mismatch bias is substantial, making the confound clearly visible.

\section{Multi-Group Extension}
\label{app:multigroup}

For $G > 2$ groups with mixture vectors $\boldsymbol{\pi}^{\mathrm{tr}}, \boldsymbol{\pi}^{\mathrm{tgt}} \in \Delta^G$:

The weighted ERM objective is:
\[
\Rhat_{\boldsymbol{\lambda}}(f) = \sum_{g=1}^G \lambda_g \Rhat_g(f) + r(f), \quad \boldsymbol{\lambda} \in \Delta^G.
\]

The bias-variance bound generalizes:
\[
\E[\Rt(\hat f_{\boldsymbol{\lambda}})] - \Rt(f^*_{\boldsymbol{\pi}^{\mathrm{tgt}}}) \leq C_{\mathrm{bias}}\|\boldsymbol{\lambda} - \boldsymbol{\pi}^{\mathrm{tgt}}\|_M^2 + C_{\mathrm{var}}\sum_{g=1}^G \frac{\lambda_g^2}{n_g},
\]
where $\|\cdot\|_M$ is induced by the gradient divergence matrix.

The shrinkage optimizer is:
\[
\boldsymbol{\lambda}^* = (1-\gamma)\boldsymbol{\pi}^{\mathrm{tgt}} + \gamma \hat{\boldsymbol{\pi}}^{\mathrm{tr}},
\]
with $\gamma$ depending on group counts and curvature constants.

\end{document}